\documentclass[submission,copyright,creativecommons]{eptcs}
\providecommand{\event}{ICLP 2025} 
\usepackage{iftex}

\ifpdf
  \usepackage{underscore}         
  \usepackage[T1]{fontenc}        
\else
  \usepackage{breakurl}           
\fi

\usepackage{listings}
\newcommand {\ri} {\rightarrow}

\newcommand {\ent} {\mathrel{{\scriptstyle\mid\!\sim}}}

\newcommand {\sx} {\langle}
\newcommand {\dx} {\rangle}

\newcommand {\emme} {\mathcal{M}}

\newcommand {\WW} {\mathcal{W}}

\newcommand{\tip}{{\bf T}}

\newcommand{\el}{\mathcal{EL}^{\bot}}

\newcommand{\be}{\begin{enumerate}}
\newcommand{\ee}{\end{enumerate}}

\newcommand{\hide}[1]{}

\usepackage{times}
\usepackage{soul}
\usepackage{url}\renewcommand\UrlFont\itshape

\usepackage{times}
\usepackage{graphicx}
\usepackage{latexsym}

\usepackage{amsmath}
\usepackage{multirow}

\usepackage{multicol,lipsum}
\usepackage{microtype}
\usepackage{amssymb}
\usepackage{color}
\usepackage{ifsym} 

\newtheorem{lemma}{Lemma}

\newtheorem{proposition}{Proposition}
\newtheorem{definition}{Definition}
\newtheorem{example}{Example}

\title{A framework for Conditional Reasoning \\ in Answer Set Programming}

\author{Mario Alviano
\institute{DEMACS, University of Calabria, Italy}
\email{mario.alviano@unical.it}
\and
Laura Giordano \quad  Daniele Theseider Dupr{\'e}
\institute{DISIT, University of Piemonte Orientale, Italy}
    \email{laura.giordano@uniupo.it \quad dtd@uniupo.it}
}

\def\titlerunning{A framework for Conditional Reasoning in Answer Set Programming}
\def\authorrunning{M. Alviano, L. Giordano and D. Theseider Dupr{\'e}}

\begin{document}

\maketitle

\begin{abstract}
In this paper we introduce a Conditional Answer Set Programming framework (Conditional ASP) for the definition of conditional extensions of Answer Set Programming (ASP).  
The approach builds on a conditional logic with typicality, and on the combination of a conditional knowledge base with an ASP program, and allows for conditional reasoning over the answer sets of the program. 
The formalism relies on a multi-preferential semantics, 
and on the KLM preferential semantics, as a special case. 
Conditional entailment is encoded in ASP and a complexity upper-bound is provided.

\end{abstract}

\section{Introduction}

Preferential logics
\cite{Delgrande:87,Makinson88,Pearl:88,KrausLehmannMagidor:90,Pearl90,whatdoes,BenferhatIJCAI93,BoothParis98,Kern-Isberner01} have been proposed to provide axiomatic foundations of non-monotonic or defeasible reasoning, and to capture commonsense reasoning.
They have their roots in conditional logics \cite{Lewis:73,Nute80},
and allow for conditional statements (simply, {\em conditionals}) of the form $\alpha \ent \beta$, meaning that ``normally if $\alpha$ holds,  $\beta$ holds". 

{\em Answer set programming} (ASP) is one of the main computational logic paradigms. It is a formalism for declarative problem solving, allowing for the declarative specification of knowledge through a set of rules. 
The semantics of an ASP program is given by assigning it a collection of {\em stable models} \cite{Gelfond&Lifschitz:88,ASP_Gebser2012}, also called {\em answer sets}.

ASP has been used for encoding entailment in some preferential logics, specifically, for preferential extensions of low complexity Description Logics (DLs) of the $\el$ family \cite{rifel}. In particular, ASP has been used for reasoning in a ranked concept-wise extension of $\el$ with typicality \cite{TPLP2020}, 
for reasoning with many-valued weighted  conditional knowledge bases \cite{JLC_SI_CILC2023,IJAR23}, 
and also for defeasible reasoning 
based on the Rational Closure \cite{TPLP2016}.

In this paper, we follow a different path. We aim at combining conditional logic and  answer set programming by pairing a conditional knowledge base $K$ with an ASP program $\Pi$, 
with the aim of using the conditional logic in the verification of conditional properties of the ASP program. In the preferential semantics, a conditional knowledge base induces a preference relation between worlds (propositional interpretations);
when combined with an ASP program $\Pi$, the preference
relation is induced on the set of {\em answer sets} of $\Pi$,
and a preferential interpretation can be constructed from such a set.
Verifying {\em conditional properties of an ASP program}, given the associated conditional KB,
is especially of interest when $\Pi$ has many answer sets, and inspecting them is not viable.

In the literature many different semantics have been considered for conditional knowledge bases, and different closure constructions for strengthening preferential and rational entailment have been explored. Among them are the well known Rational Closure \cite{whatdoes} and Lexicographic Closure \cite{Lehmann95}. 
They exploit a single preference relation, while, in the paper, we explore an approach based on {\em multiple preference relations}, which has been considered in conditional description logics with typicality for ranked knowledge bases 
\cite{TPLP2020}, and for weighted (many-valued) knowledge bases \cite{JLC_SI_CILC2023,IJAR23}.

As in the Propositional Typicality Logic \cite{BoothCasiniAIJ19} (and in some DLs with typicality  \cite{FI09}), we adopt a language in which conditionals are formalized based on material implication plus the {\em typicality operator} $\tip$.
The typicality operator allows for the definition of {\em conditional implications} 
$\tip(\alpha) \rightarrow  \beta$, which correspond to the KLM conditionals $\alpha \ent \beta$ \cite{KrausLehmannMagidor:90,whatdoes}, meaning that ``in the typical situations in which $\alpha$ holds, $\beta$ also holds''.
For instance, the conditional implication:
\begin{quote}
$ \tip( \mathit{granted\_Loan}) \ri \mathit{living}\_\mathit{in}\_\mathit{town} \wedge \mathit{young}$
\end{quote}
means that, normally, if a loan is granted to a person, she lives in town and is young.
The typicality operator allows for more general implications of the form $\alpha \ri \beta$, 
where $\tip$ may occur in both $\alpha$ and $\beta$.
For instance, the converse implication
$\mathit{living}\_\mathit{in}\_\mathit{town} \wedge \mathit{young} \ri \tip( \mathit{granted\_Loan})$ 
means that living in town and being young implies that normally the loan is granted.

The preferential semantics of the logic exploits {\em multiple preference relations} $<_A$ with respect to different distinguished propositions $A$, following the approach developed for ranked and weighted KBs in description logics, based on a {\em multi-preferential semantics} \cite{TPLP2020,JLC_SI_CILC2023} 
and in some refinements of the KLM logics \cite{AIJ21}, where preferences are allowed with respect to different aspects.

In the paper, we let a {\em Conditional ASP (CondASP) program} to be a pair $(\Pi,K)$, where $\Pi$ is an ASP program and $K$ is a {\em  defeasible knowledge base}. We consider the special case where $K$ is a  
weighted KB, in which distinguished propositions $A_i$ are associated with a set of conditionals, each one with a weight, representing the plausibility/implausibility of a defeasible property.

For instance, consider the following example, in which defeasible properties of students and employees are specified through a set of weighted conditionals, as follows:
 \begin{quote}
$(d_1)$ $\mathit{\tip(employee) \ri young}$, \ \ -50  \ \ \ \ \ \ \ \ \  \ \ \ \ \ \ \ \ \ \ \ \ \ \ 
$(d_4)$ $\mathit{\tip(student) \ri young}$, \ \ 90

$(d_2)$ $\mathit{\tip(employee) \ri  has\_boss}$, \ \ 100 \ \ \ \ \ \ \ \ \ \ \ \ \ \ \ \ \ \ 
$(d_5)$ $\mathit{\tip(student) \ri   has\_boss}$, \ \ -80

$(d_3)$ $\mathit{\tip(employee) \ri   has\_classes}$, \ \ -70 \ \ \ \ \ \ \ \ \ \ \ \ \ \  
$(d_6)$ $\mathit{\tip(student) \ri   has\_classes}$, \ \  80

\end{quote}
Conditionals $(d_1 \-- d_3)$ describe properties of typical employees,
who have a boss, and are not likely to be young nor to have classes;
$(d_4 \-- d_6)$ describe properties of typical students.
Absolute values of weights correspond to degrees of (im)plausibility.
Consider a situation $S_1$ describing  a student and employee who is  young, has no boss and has classes, and a situation $S_2$ describing a student and employee who is not young, has a boss and has no classes.
Given the weights above, one would regard
$S_1$ as a more plausible situation for a student than $S_2$, i.e., $S_1 <_{\mathit{student}} S_2$;
and $S_2$ as a more plausible situation for an employee than $S_1$, i.e.,
$S_2 <_{\mathit{employee}} S_1$.

In the paper, after defining (Section \ref{sec:conditional_logic})
a {\em two-valued and multi-preferential conditional logic with typicality}, which is a generalization of KLM preferential semantics, 
we develop (Section \ref{sec:condASP}) a construction for defining, for a set of distinguished propositions $A_i$, preference relations $\leq_{A_i}$, preorders on the answer sets of the program.
In general, alternative constructions 
may be viable for defining the relations $\leq_{A_i}$, depending on the definition of the conditional knowledge base (whether it contains conditionals having a rank, or a weight, or none). 
Although the proposed approach provides a framework for conditional reasoning in ASP under different preferential interpretations of conditionals, the paper focuses on weighted KBs,
as the basis for the conditional extension of ASP.
An extended example of commonsense reasoning over a CondASP program is provided in Section \ref{sec:condASP}.

To deal with the verification of general conditionals  $\tip(A) \rightarrow  B$ over an CondASP program,   
the formalism is complemented with an approach for {\em combining preferences},
to define a preference relation $\leq_{A}$ with respect to a {\em boolean combination} $A$
of the distinguished propositions $A_i$.
This is discussed in Section \ref{sec:combining_preferences}, where
it is proved that 
entailment of conditional implications $\tip(A) \rightarrow  B$
from a CondASP program satisfies the {\em KLM properties} of preferential entailment \cite{KrausLehmannMagidor:90}.
Section \ref{sec:verif} provides an ASP approach for verifying
entailment from a CondASP program, which also provides a {\em complexity upper-bound} for conditional entailment.


\section{A Multi-Preferential Logic with Typicality} \label{sec:conditional_logic}

{In this section we define a two-valued preferential logic with typicality, which generalizes Kraus Lehmann and Magidor preferential semantics \cite{KrausLehmannMagidor:90,whatdoes}, 
by allowing for multiple preference relations (i.e., preferences with respect to multiple aspects), rather that a single preference relation.

Let ${\cal L}$ be a propositional logic, whose formulae are built from a set $Prop$ of propositional variables using the boolean connectives $\wedge$, $\vee$, $\neg$ and $\rightarrow$, as usual. 
We assume that $\bot$ (representing falsity) and $\top$ (representing truth) are formulae of ${\cal L}$.

We extend ${\cal L}$ by introducing a typicality operator as done in \cite{lpar2007} and \cite{BoothCasiniAIJ19} for the propositional calculus.
Intuitively, ``a sentence of the form $\tip(\alpha)$ is understood to refer to the {\em typical situations in which $\alpha$ holds}" \cite{BoothCasiniAIJ19}.
The typicality operator allows for the formulation of  {\em  conditional implications} (or {\em defeasible implications}) of the form 
$\tip(\alpha) \rightarrow \beta$
whose meaning is that ``normally, if $\alpha$ then $\beta$'', 
or ``in the typical situations when $\alpha$ holds, $\beta$ also holds''.
They correspond to conditional implications $\alpha \ent \beta$ of KLM preferential logics \cite{whatdoes}, but here we consider a preferential semantics which exploits multiple preferences. 
As in PTL \cite{BoothCasiniAIJ19}, the typicality operator cannot be nested.
When $\alpha$ and $\beta$ do not contain occurrences of the typicality operator, an implication $\alpha \rightarrow \beta$ is called {\em strict}. 
We call ${\cal L}^\tip$ the logic obtained by extending  ${\cal L}$ with a unary typicality operator $\tip$. 
In the logic ${\cal L}^\tip$,  {\em general implications}  $\alpha \rightarrow \beta$ are allowed, where $\alpha$ and $\beta$ may contain occurrences of the typicality operator.

As mentioned above, the interpretation of a typicality formula  $\tip(\alpha)$ is defined with respect to a {\em multi-preferential interpretation}.
The KLM preferential semantics \cite{KrausLehmannMagidor:90,whatdoes,Pearl:88} exploits a set of worlds $\WW$, with their valuation and a preference relation $<$ among worlds (where $w < w'$ means that world $w$ is more normal than world $w'$). 
A conditional $A\ent B$ is  satisfied in a KLM preferential interpretation, if $B$ holds in all the most normal worlds satisfying $A$, i.e., in all $<$-minimal worlds satisfying $A$.

Here, instead, we consider a multi-preferential semantics, where preference relations are associated with distinguished propositional variables $A_1, \ldots, A_n \in\mathit{Prop}$ (called {\em distinguished propositions} in the following). The idea is that how much a situation (a world) is normal (or less atypical) with respect to another one depends on the aspects considered for comparison.
The semantics exploits a set of preference relations $<_{A_i}$, each associated to a distinguished proposition $A_i$, where $w <_{A_i} w'$ means that world $w$ is less atypical than world $w'$ concerning aspect $A_i$.
For instance, referring to the example in the introduction, a world $w$ can be regarded as capturing a more plausible situation describing a student than world $w'$ (i.e., $w <_\mathit{student} w'$),
but a less plausible situation than $w'$ describing an employee (i.e., $w' <_\mathit{employee} w$).

Multi-preferential semantics have been previously considered for defining refinements of the rational closure construction \cite{AIJ21}, 
for defining preferences with respect to different modules of a conditional knowledge base \cite{NMR2020}, for defining a concept-wise semantics in ranked defeasible $\el$ knowledge bases \cite{TPLP2020}, and weighted conditionals in {\em many-valued} DLs with typicality \cite{IJAR23}.  Here, we deal with a {\em two-valued} multi-preferential logic with typicality.

Given a finite set of {\em distinguished propositions} $A_1, \ldots, A_n$, we let  $\leq_{A_i} \subseteq W\times W$ be the {\em preorder} associated with $A_i$ ($w \leq_{A_i} w'$ means that $w$ is not less plausible than $w'$). 
A strict partial order $<_{A_i}$ and an equivalence relation $\sim_{A_i}$ can be defined as usual from the preorder $\leq_{A_i}$ by letting 
$w <_{A_i} '$ iff $w \leq_{A_i} w'$ and  $w' \not \leq_{A_i} w$, and letting $w \sim_{A_i} w'$ iff $w \leq_{A_i} w'$ and  $w' \leq_{A_i} w$.

\subsection{Multi-preferential semantics}

In the following, we  limit our consideration to finite KBs, and restrict our attention to preferential interpretations with a finite set of preference relations $<_{A_i}$, one for each distinguished proposition  $A_i$. 
For the moment, we assume that, in any typicality formula $\tip(A)$, $A$ is a distinguished proposition.
In Section \ref{sec:combining_preferences} we will lift this restriction.

\begin{definition}\label{MPinterpretations}
A {\em (multi-)preferential  interpretation}  is a triple $\emme= \sx \WW, \{<_{A_i}\}, v \dx$ where:
\begin{quote} 
$\bullet$ \ $\WW$ is a non-empty set of worlds;

$\bullet$ \  each $<_{A_i} \subseteq \WW \times \WW$ is an irreflexive and transitive relation on $\WW$; 

$\bullet$ \ $v: \WW \longrightarrow 2^\mathit{Prop} $ is a valuation function, assigning to each world $w$ a set of propositional variables in $\mathit{Prop} $, the variables which are true in $w$.
\end{quote}
\end{definition} 
\noindent
A \em ranked interpretation} is a (multi-)preferential interpretation  $\emme= \sx \WW,\{<_{A_i}\}, v \dx$ for which all preference relations $<_{A_i}$ are {\em modular}, that is: for all $x, y, z$, if
$x < _{A_i} y$ then $x <_{A_i} z$ or $z <_{A_i} y$ (the preorder $\leq_{A_i}$ is total).
A relation $<_{A_i}$  is {\em well-founded} if it does not allow for infinitely descending chains of worlds $w_0, w_1, w_2, \ldots$. with $w_1 <_{A_i} w_0$, \  $w_2 <_{A_i} w_1$, $\ldots$.

The valuation $v$ is inductively extended to  all formulae of ${\cal L}^\tip$: 
\begin{quote}
 $\emme, w \models \top$  \ \ \ \ \ \ \ \ \ \ \ \ \ \ \ \ \ \ \ \ \  $\emme, w \not \models \bot$
 
 $\emme, w \models p$ \  iff 	 $p \in  v(w)$, \ \ for all $p \in \mathit{Prop}$

$\emme, w \models A \wedge B$ iff $\emme, w \models A $ and $\emme, w \models B$

$\emme, w \models A \vee B$ iff $\emme, w \models A $ or $\emme, w \models B$

$\emme, w \models \neg A$ iff $\emme, w \not \models A $ 

$\emme, w \models A \ri B$ iff $\emme, w \models A $ implies $\emme, w \models B$

$\emme, w \models \tip(A_i)$ iff $  \emme, w \models A_i $ and $ \mathit{\nexists w' \in \WW}$ \ s.t.$ \mathit{ w' <_{A_i} w} \mbox{  and } \emme, w' \models A_i.$

\end{quote} 
Whether $\tip(A_i)$ is satisfied at a world $w$  also depends on the other worlds of the interpretation $\emme$.

Let $ [[A]]^ \emme$ be the set of all the worlds in $\emme$ satisfying a formula $A$ (i.e., $ [[A]]^ \emme=\{w \in \WW : \emme,w \models A \} $)
and let $Min_{<_{A}}(\mathcal{S})$ be the set of $<_{A}$-minimal worlds in $\mathcal{S}$, for any set of worlds $\mathcal{S} \subseteq \WW$, and strict partial order $<_A$, that is:
\begin{quote}
$Min_{<_{A}}(\mathcal{S})= \{ w \in \mathcal{S} \mid   \mbox{ there is no } w' \in \mathcal{S},  \mbox{ such that } w' <_{A} w   \}$
\end{quote}
Note that, when $\mathcal{S} \neq \emptyset$ and the preference relation $<_{A} $ is well-founded, the set $Min_{<_{A}}(\mathcal{S})$ is non-empty. For a {\em well-founded preference relation} $<_{A} $,
one can reformulate the semantic condition for the typicality operator as follows:

\begin{quote}
$\emme, w \models \tip(A_i)$ \  iff \  $w \in Min_{<_{A_i}}([[A_i]]^\emme)$.
\end{quote}

We say that a formula  {\em $A$ is satisfiable} in the multi-preferential semantics if there exist a multi-preferential interpretation $\emme= \sx \WW, \{<_{A_i}\}, v \dx$  and a world $w \in \WW$ such that $\emme,w \models A$.
A formula {\em $A$ is  valid in an interpretation $\emme$} (written $\emme \models A$) if, for all worlds $w \in \WW$, $\emme,w \models A$.
A formula $A$ is {\em valid} in the multi-preferential semantics (simply, {\em A is valid}) if $A$ is valid in any multi-preferential interpretation $\emme$.
Restricting our consideration to ranked interpretations, leads to the notions of satisfiability and validity of a formula in the {\em ranked (or rational) multi-preferential semantics}.

As mentioned above, when a defeasible implication has the form $\tip(A)  \rightarrow B$, with $B$ in ${\cal L}$, it stands for a conditional  $A \ent B$ in KLM logics \cite{KrausLehmannMagidor:90}. 
Note that, if we assume that all preference relations are well-founded, a defeasible implication $\tip(A)  \rightarrow B$ is valid in a preferential interpretation $\emme$ (i.e., $\emme \models \tip(A)  \rightarrow B$) {\em iff} for all worlds $w \in \WW$, $w \in Min_{<_{A}}([[A]]^\emme)$ implies $w \in [[B]]^\emme$, i.e., {\em iff} $Min_{<_{A}}([[A]]^\emme) \subseteq [[B]]^\emme$ holds. 
In this case, if we assume that all the preference relations  $<_{A}$ coincide with a single well-founded preference relation $<$, then a multi-preferential interpretation $\emme$ becomes a preferential interpretation as in KLM semantics: 
\begin{center}
$\emme \models \tip(A)  \rightarrow B$ 
{\em iff} $Min_{<}([[A]]^\emme) \subseteq [[B]]^\emme$,
\end{center}
i.e., $\emme \models \tip(A)  \rightarrow B$ holds when  the conditional formula $A \ent B$ is satisfied in the preferential model $\emme$ according to the preferential KLM semantics \cite{KrausLehmannMagidor:90} (resp., to  the rational KLM semantics, when $<$ is assumed to be modular). 
The multi-preferential semantics is indeed a generalization of the KLM preferential semantics.

Let a {\em knowledge base $K$} be a set of (strict or defeasible) implications.
A {\em preferential model of $K$} is a multi-preferential interpretation $\emme$ 
such that $\emme \models A \ri B$, for all implications $A \ri B$ in $ K$.
Given a knowledge base $K$,  we say that an implication {\em  $A  \rightarrow B$ is  preferentially entailed from}  $K$  \ if \ $\emme \models A  \rightarrow B $ holds, for all preferential models $\emme$ of $K$. We say that  {\em  $A  \rightarrow B$ is rationally entailed from}  $K$  \ if \ $\emme \models A  \rightarrow B $ holds, for all ranked models $\emme$ of $K$.

It is well known that preferential entailment and rational entailment are weak. As with the rational closure \cite{whatdoes} and the lexicographic closure \cite{Lehmann95} for KLM conditionals, also in the multi-preferential case one can strengthen entailment by restricting to specific preferential models,  
based on some {\em closure constructions}, which allow to define the preference relations $<_{A_i}$ from a knowledge base $K$, also exploiting the ranks and weights of conditional implications, when available.
Some examples of closure constructions for the multi-preferential case have been considered,
e.g., for variants of the rational closure \cite{AIJ21} and of the lexicographic closure \cite{NMR2020}, and for ranked or weighted defeasible DLs with typicality \cite{TPLP2020,IJAR23}. 

In the next section, we define an approach for combining an ASP program with a  conditional knowledge base, and develop a construction based on the weighted conditional KB for reasoning on the answer sets of the ASP programs.


\section{A conditional extension of Answer Set Programming}\label{sec:condASP}

In this section, we aim at combining a defeasible knowledge base with an ASP Program, to exploit conditionals to define preferences between the Answer Sets of the ASP program. Based on such preferences, conditional properties of the program can be validated over a preferential interpretation of the program, based on the semantics of the conditional logic with typicality introduced in the previous section. 

More precisely, a {\em Conditional ASP (CondASP) program} is a pair $(\Pi, K)$, where $\Pi$ is an ASP program and $K$ is a defeasible knowledge base. Both $\Pi$ and $K$ rely on the same set of propositional variables $\mathit{Prop}$. 
In this paper we specifically consider weighted knowledge bases, and 
assume that $K$ is a weighted KB, i.e., a set of conditional implications, with an associated integer weight. 

For simplicity, let $A_1, \ldots, A_n$ be a set of distinguished propositional variables.
A {\em weighted defeasible knowledge base} $K$ over the distinguished propositions $A_1, \ldots, A_n$ is a set of  {\em defeasible implications}
of the form $(\tip(A_i) \sqsubseteq B_{i,j}, h_{i,j})$, where $A_i$ is a distinguished proposition, $B_{i,j} \in \mathcal{L}$ and the weight $h_{i,j}$ is an integer number.
The weight is intended to represent the {\em plausibility} (or {\em implausibility}) of the associated defeasible implication, where negative weights represent penalties.

Let us consider the following example program, which is inspired by the cinema scenario in \cite{Brewka04}.
On Saturday evening, two married persons, as Bob and Mary, can either go to the cinema, go for a pizza or stay at home.
But they cannot both go out if they have children and a trusted babysitter is not available.
\smallskip
\begin{quote}

$ \mathit{
1 \{ go\_to\_cinema(X); go\_for\_pizza(X); stay\_at\_home(X) \} 1
  \ \text{:--} \ person(X). 
} $

$ \mathit{
1 \{ stay\_at\_home(X); stay\_at\_home(Y) \}  \ \text{:--} \ married(X,Y), 
   have\_children(X,Y), 
} $

\ \ \ \ \ $ \mathit{
   not \; available\_babysitter\_for(X,Y).
} $

$ \mathit{ available\_babysitter\_for(X,Y)  \ \text{:--} \ available\_babysitter\_for(X,Y,Z).
} $

$ \mathit{ available\_babysitter\_for(X,Y,Z)  \ \text{:--} \ married(X,Y), babysitter(Z), 
} $

\ \ \ \ \  $ \mathit{  not \ busy(Z), 1 \{ trust(X,Z); trust(Y,Z) \}.
} $

$ \mathit{  \ \text{:--} \ go\_to\_cinema(Z), available\_babysitter\_for(X,Y,Z).
} $

$ \mathit{  \ \text{:--} \  go\_for\_pizza(Z), available\_babysitter\_for(X,Y,Z).
} $

$ \mathit{ \{ housekeeping(X) \}  \ \text{:--} \ stay\_at\_home(X). 
} $

$ \mathit{ busy(X)  \ \text{:--} \ go\_to\_cinema(X).
} $
\ \ \ \ \ \ \ \ \ \ \ \  \ \ 
$ \mathit{ busy(X)  \ \text{:--} \ go\_for\_pizza(X).
} $

$ \mathit{married(bob, mary).
} $  
\ \ \ \ \ \ \ \ \ \ \ \ \ \ \ \ \ \ \ \ \ \ \ \ \ \ \ \ \ \ $ \mathit{have\_children(bob, mary). 
} $

$ \mathit{babysitter(ada).
} $  
\ \ \ \ \ \ \ \ \ \ \ \ \ \ \ \ \ \ \ \ \ \ \ \ \ \ \ \ \ \ \ \ \ \ \ \ \ \ $ \mathit{trust(mary, ada). 
} $

$ \mathit{person(X)  \ \text{:--} \ babysitter(X).
} $

$ \mathit{person(X)  \ \text{:--} \ married(X,Y).
} $
\ \ \ \ \ \ \ \ \ \ 
$ \mathit{person(Y)  \ \text{:--} \ married(X,Y).
} $

$ \mathit{\text{:--} \ trust(mary,Z), Z!=ada.
} $

$ \mathit{\text{:--} \ go\_for\_pizza(mary), not \ go\_for\_pizza(bob).
} $

$ \mathit{ \text{:--} \ go\_for\_pizza(bob), not \ go\_for\_pizza(mary).
} $

$ \mathit{ \text{:--} \ happy\_Sat(bob), housekeeping(bob).
} $

$ \mathit{ \{happy\_Sat(X) \}  \ \text{:--} \ person(X).
} $
\end{quote}

\smallskip
\noindent
We have assumed that Mary and Bob would only have a pizza together, but they can go to the cinema independently.
Let us assume that Mary, Bob and Ada have their preferences for Saturday evening, which are expressed by the following weighted knowledge base.
\begin{quote}
$\tip(happy\_Sat(mary)) \ri go\_to\_cinema(mary), +40$

$\tip(happy\_Sat(mary)) \ri stay\_at\_home(mary), -10$

$\tip(happy\_Sat(mary)) \ri go\_for\_pizza(mary), +20$

$\tip(happy\_Sat(bob)) \ri go\_to\_cinema(bob), +10$

$\tip(happy\_Sat(bob)) \ri stay\_at\_home(bob), -30$

$\tip(happy\_Sat(bob)) \ri go\_for\_pizza(bob), +40$

$\tip(happy\_Sat(ada)) \ri go\_to\_cinema(ada), +30$

$\tip(happy\_Sat(ada)) \ri go\_to\_cinema(mary), +30$
\end{quote}

\noindent
While Mary has a preference for going to the cinema, Bob is more fond of pizza. Both of them dislike staying at home on the Saturday evening (and, from the ASP program above, we know that Bob would not have a happy Saturday doing housekeeping).
Concerning Ada's preferences, she likes going to the cinema.
However, as she is Mary's friend, she is equally happy when Mary can go to the cinema on Saturday.

Which are the typical situations in which Mary is happy, or Mary and Bob are happy on Saturday evening? And what does also hold in such situations?
We aim at using the conditional logic to answer such questions. 
For instance, considering the typical situations in which Bob and Mary are happy on Saturday, are these situations where they go for pizza, i.e.,
$\tip(happy\_Sat(mary) \wedge happy\_Sat(bob) ) \ri go\_for\_pizza(mary) \wedge go\_for\_pizza(bob) $?
Further, are these typical situations in which Mary is happy on Saturday, i.e.,
$\tip(happy\_Sat(mary) \wedge happy\_Sat(bob) ) \ri \tip(happy\_Sat(mary))$?

In the following we describe how a weighted knowledge base can be used to define preferences over the answer sets of an ASP program (see e.g.\ \cite{ASP_Gebser2012} for the definition of the semantics of ASP programs including constraints and choice rules used in the example above).

\subsection{Weighted knowledge bases and preferences over Answer Sets} \label{sec:weighted_KBs}

Given a weighted CondASP program $(\Pi,K)$, if $\Pi$ has answer sets,
we can construct a multi-preferential interpretation ${\emme^\Pi_K}$ such that the set of worlds $\WW$ is the set of the answer sets of the program $\Pi$,  
the valuation function maps each answer set $S$ into the corresponding propositional interpretation, 
and the preference relations $<_{A_i}$ on the answer sets of $\Pi$ are defined from the set of weighted conditionals in $K$ having the form $(\tip(A_i) \ri B_{i,j} \; , h_{i,j})$.

As an answer set $S$ of an ASP program $\Pi$ is a set of propositional variables in $\mathit{Prop}$, it can be regarded as a propositional interpretation (in which the variables in $S$ are true, and the others are false), and the satisfiability in $S$ of any boolean formula is defined as usual. The preference relation $<_{A_i}$ associated with the distinguished proposition $A_i$ is determined by computing, for each answer set $S \in \WW $, the weight $W_{A_i}(S)$ of $S$ with respect to proposition $A_i$:  
\begin{align}\label{weight}
	W_{A_i}(S)  = \sum_{{j: S \models B_{i,j}}} h_{i,j}  
\end{align}
Informally, the weight $W_{A_i}(S)$ of $S$ wrt $A_i$ is the sum of the weights of all defeasible implications $\tip(A_i) \ri B_{i,j}$ for $A_i$, such that $B_{i,j}$ is satisfied by $S$. The more plausible are the properties satisfied by $S$ in the typical situations for $A_i$,  the higher is the weight of $S$ with respect to $A_i$. 
If no defeasible implication $(\tip(A_i) \ri B_{i,j} \; , h_{i,j})$ is in $K$ for a distinguished proposition $A_i$, then $W_{A_i}(S)=0$ holds for all answer sets $S$. 
The preorder relation $\leq_{A_i}$ associated to proposition $A_i$ is defined as follows:
\begin{align}  \label{pref}
S_1 & \leq_{A_i}  S_2  \mbox{  \ \ iff \ \ }  W_{A_i}(S_1) \geq W_{A_i}(S_2).
\end{align}

A CondASP program $(\Pi,K)$ is {\em consistent} if $\Pi$ has at least an answer set.
Given a consistent CondASP program $(\Pi,K)$, we define a unique multi-preferential model of $(\Pi,K)$, as follows.
\begin{definition}
The {\em multi-preferential model} of a consistent CondASP program $(\Pi,K)$ is a preferential interpretation ${\emme^\Pi_K}= \sx \WW, \{<_{A_i}\}, v \dx$ s.t.:
\begin{quote}
$\bullet$ \ $\WW$ is the set of all the answer sets of $\Pi$;

$\bullet$ \ for all $S \in \WW$, $v(S)=S$;

$\bullet$ \ for all distinguished propositions $A_i$, 
 $<_{A_i}$ is  the strict relation induced by the preorder $\leq_{A_i}$, that is: 
for all $S_1, S_2 \in \WW$, $S_1 <_{A_i}  S_2$    \ \ iff \ \  $W_{A_i}(S_1) > W_{A_i}(S_2)$.
\end{quote}
\end{definition}
An answer set $S_1$ is {\em preferred to an answer set} $S_2$ with respect to $A_i$ (written $S_1 <_{A_i} S_2$), if the weight of $S_1$ with respect to $A_i$ is greater than the weight of $S_2$ with respect to $A_i$; and vice-versa. 

We regard  ${\emme^\Pi_K}$ as the {\em canonical model} of $(\Pi,K)$.\footnote{Restricting to preferential models over a set of propositional interpretations is a standard approach in KLM logics, e.g., in the semantics of rational closure and of lexicographic closure of a KB \cite{whatdoes,Lehmann95}.}.
To clarify  the definition of preferences, let us consider again the example about students and employees, from the introduction.

\begin{example}  \label{exa:student2}
Consider a CondASP program $(\Pi,K)$, such that, for instance:

$S_1=\{student, employee, adult, has\_SSN,$ $ young, has\_classes\}$ 

$S_2 =\{student, employee, adult, has\_SSN, has\_boss\}$ 

\noindent
are both answer sets of $\Pi$, and $K$ consists of the defeasible properties $(d_1)-(d_6)$ from the introduction.
The weight of $S_1$ with respect to proposition $\mathit{student}$ is $W_\mathit{student}(S_1)= 90+80=170$,
while the weight of $S_2$ with respect to proposition $\mathit{student}$ is $W_\mathit{student}(S_2)= -80$.
Hence, $S_1 <_{\mathit{student}} S_2$.
On the other hand, $W_\mathit{employee}(S_1)= -120$, and $W_\mathit{employee}(S_2)= 100$.
Hence, $S_2 <_{\mathit{employee}} S_1$.
\end{example}

Given a consistent CondASP program  $(\Pi,K)$, and an implication $A  \rightarrow B$, 
we say that  $A  \rightarrow B$  {\em  is entailed from the program $(\Pi,K)$} 
(written $(\Pi,K) \models  A  \rightarrow B$), \ if \ ${\emme^\Pi_K} \models A  \rightarrow B $ holds. 

\subsection{Other preference definitions: Preferences based on the ranks of conditionals}

When the defeasible knowledge base is a set of conditional implications
$\tip(A) \ri B$, with $A,B \in {\cal L}$, computing the Rational Closure or the Lexicographic closure of the knowledge base are viable options to define a single preorder relation $\leq$ 
on the answer sets of the program $\Pi$. In this case, the ranks can be computed according to the rational closure construction, e.g., using Datalog with stratified negation, as in the Datalog encoding of the rational closure of the lightweight description logic $\cal SROEL$ \cite{FI_2018_sroel}.

In the previous section, we have seen that the {\em (user defined) weights} 
can be exploited to define the different preorder relations $\leq_{A_i}$ for the distinguished propositions $A_i$.
Another option is that {\em (user defined) ranks} $0,1,2, \ldots$ are associated to the conditional implications in the knowledge base $K$, and that conditionals with higher ranks are considered to be more important than conditionals with a lower ranks. As for  weights, one can use the ranks associated to the conditional implications $\tip(A_i) \ri B_{i,j} $ for a distinguished proposition $A_i$, to define the preorder relation $\leq_{A_i}$.
This way to define the preferences relations has been exploited, for instance, in a defeasible description logic with typicality \cite{TPLP2020}, in which preferences are associated to distinguished concepts $C_i$,
but the idea of having ranked knowledge bases with user-defined ranks was previously explored in Brewka's framework for qualitative preferences \cite{Brewka89,Brewka04}, in which basic preference relations $\geq_K$  are associated with different {\em ranked knowledge bases} $K$. 

In the general case, an interesting option can be allowing different ways of specifying preferences for different distinguished propositions $A_i$ {\em in the same knowledge base}. 
As an example, let us define the typical properties of {\em horses}, by a set of conditional implications with a rank:
\begin{quote}
$(d_7)$ $\mathit{\tip(horse) \ri has\_Saddle}$, \ \ 0 
\ \ \ \ \ \ \ \ \ \ \ \ \ \ \ \ \ \ \ \ \ \ \ 
$(d_9)$ $\mathit{\tip(horse) \ri   run\_fast}$, 1

$(d_8)$ $\mathit{\tip(horse) \ri   has\_Long\_Mane}$, 0
\ \ \ \ \ \ \ \ \ \ \ \ \ \ \ \ \ 
$(d_{10})$ $\mathit{\tip(horse) \ri   has\_tail}$, 1
\end{quote}
where the properties of having a tail and running fast are regarded as being more important (for a horse) than the properties of having a saddle and a long mane. 

Adopting the definition in \cite{TPLP2020}, 
one can define the preference relation $\leq_{A_i}$  associated with distinguished proposition $A_i$,considering the ranks of the defeasible implications of the form $\tip(A_i) \ri B_{j,i}$ in $K$, by letting ${\cal T}_{A_i}^l(S)$ be the set of typicality inclusions for $A_i$ with rank $l$, being  satisfied by $S$, that is:
\begin{quote}
${\cal T}_{A_i}^l(S) = \{\tip(A_i) \ri B_{j,i} \mid (\tip(A_i) \ri B_{j,i},l) \in K \mbox{ and } S \not\models A_i \mbox{ or } S\models B_{j,i}\}.$
\end{quote}
\noindent
Given two answer sets $S_1$ and $S_2$ of the program $\Pi$, we let
\begin{quote}
$S_1  \leq_{A_i}  S_2  \mbox{  \ \ iff \ \ }   \mbox{either }  |{\cal T}_{A_i}^l(S_1) |=  |{\cal T}_{A_i}^l(S_2)|$, for all $l$, 

 $\mbox{ \ \ \ \ \ \ \ \ \ \ \ \ \ \
or  $\exists l$ such that $|{\cal T}_{A_i}^l(S_1)| > |{\cal T}_{A_i}^l(S_2)|$ and, $\forall h>l$, $|{\cal T}_{A_i}^h(x_1)| = |{\cal T}_{A_i}^h(S_2)|$.} $
\end{quote}
Informally, the preference relation $\leq_{A_i}$ gives higher preference to answer sets violating a smaller number of conditional implications with higher rank for $A_i$. It corresponds to
the strategy $\#$ in Brewka's framework for qualitative preferences \cite{Brewka04}, which defines a total preorder.
We refer to  \cite{TPLP2020} for an encoding in {\em asprin} \cite{BrewkaAAAI15} of the preferences $\leq_{A_i}$, used therein  for reasoning from a ranked $\mathcal{SROEL}$ knowledge base.

In a general framework, the ranked conditionals $(d_7) -(d_{10})$ specifying the preorder 
$\mathit{\leq_{horse}}$, can belong to the same knowledge base $K$ as the weighted conditionals $(d_1)- (d_6)$ specifying the preference relations $\mathit{\leq_{student}}$ and $\mathit{\leq_{employeed}}$.
However, in the following, we will restrict our attention to knowledge bases only containing weighted conditionals, and develop an ASP encoding of entailment from a CondASP program $(\Pi,K)$, with $K$ a weighted KB.

Let us first consider the issue of preference combination for a CondASP program $(\Pi,K)$, with $K$ a weighted conditional KB.


\section{Combining preferences from weighted knowledge bases} \label{sec:combining_preferences}

In previous approaches to multi preferential semantics in the two-valued case, e.g., \cite{AIJ21,NMR2020,TPLP2020},
a global preference relation $<$ is defined from the base preferences $<_{A_i}$;
this allows to provide a meaning of a defeasible implication $\tip(A)  \rightarrow B$, for any formula $A$, by letting:
\begin{center}
$\emme \models \tip(A)  \rightarrow B$ 
{\em iff} $Min_{<}([[A]]^\emme) \subseteq [[B]]^\emme$,
\end{center}

However, choosing a single global preference relation $<$ may be restrictive, as one may want to exploit different minimization criteria and allow for different user-defined preference relations.
In this paper, rather than defining a single global preference relation $<$, we  allow for the definition of different preference relations, associated with complex formulae obtained by combining the distinguished propositions $A_1, \ldots, A_n$.
For instance, referring to the example in Section \ref{sec:condASP}, we aim at considering the most plausible situations in which both Mary and Bob are happy on Saturday.

In Brewka's framework for preference combination \cite{Brewka04}, a logical preference description language is introduced for combining basic preference descriptions $d_1$ and $d_2$ (defining preorders on models) into complex preference descriptions 
$d_1 \wedge d_2$, $d_1 \vee d_2$, $\neg d_1$ and $d_1 >d_2$ (where $d_1 > d_2$ is intended to express the priority of preference $d_1$  over preference $d_2$).
However, as observed by Brewka \cite{Brewka04}, the preference description $\neg (d_1 \vee d_2)$ is different from $\neg d_1 \wedge \neg d_2$. This would lead to break some wanted KLM properties of a nonmonotonic consequence relation \cite{KrausLehmannMagidor:90}, such as {\em Left Logical Equivalence} (see below).

In the following, for weighted KBs, we propose an alternative approach for preference  combination, which exploits the weights $W_{A_i}(S)$, introduced to define the preference relations $<_{A_i}$ in $\emme^\Pi_K$.
We first extend, in Definition \ref{MPinterpretations}, the semantic condition
for the typicality operator to all typicality formulae $\tip(A)$ such that $A$ is a boolean combination of $A_1, \ldots, A_n$ (i.e., $A$ is obtained by combining the distinguished propositions $A_i$ using $\wedge$, $\vee$ and $\neg$, while $\bot$ and $\top$ are not allowed to occur in $A$). As the set $\WW$ of answer sets of $\Pi$ is finite, 
we can restrict to well-founded preference relations, and let:
\begin{center}
${\emme^\Pi_K}, S \models \tip(A)$ \  iff \  $S \in Min_{<_{A}}([[A]]^{{\emme^\Pi_K}})$.
\end{center}
For evaluating a typicality formula $\tip(A)$ at a world, we need to define a preference relation $\leq_A$ associated to a formula $A$. 
For instance, for determining the most normal situations in which both Bob and Mary are happy on Saturday ($\mathit{\tip(happy\_Sat(bob) \wedge happy\_Sat(mary))}$ we need to define the preference relation $\mathit{<_{happy\_Sat(bob) \wedge happy\_Sat(mary)}}$.

To define the preorder relation $\leq_A$ associated to a complex formula $A$ (a boolean combination of the $A_i$'s) we inductively extend to complex formulae the notion of the weight of an answer set with respect to a formula.
Given $(\Pi,K)$, let $\mathit{Max}$ and $\mathit{Min}$ be,
resp., the maximum and the minimum value of the weight $W_{A_i}(S)$, for each distinguished proposition $A_i$ in $K$ and  answer set $S$. 
We let:
\begin{quote}
$W_{A_1 \wedge A_2}(S)= min( W_{A_1}(S), W_{A_2}(S))$
\ \ \ \ \ \ \ \ \ \ \ 
$W_{A_1 \vee A_2}(S)= max( W_{A_1}(S), W_{A_2}(S))$

$W_{\neg A_i}(S)= Max - W_{A_i}(S) + Min$
\end{quote}
\begin{proposition}
The following properties hold:
\begin{quote}
$W_{\neg \neg A}(S) = W_A(S)$ 
\ \ \ \ \ \ \ \ \ \ \ \ \ \ \ \ \ \ \ \ \ \ \ \ \ \

$W_{\neg (A \wedge B)}(S) = W_{\neg A \vee \neg B}(S)$
\ \ \ \ \ \ \ \ \ \ \ \ \ \ \ \ \ \ \  $W_{\neg (A \vee B)}(S) = W_{\neg A \wedge \neg B}(S)$

$W_{A \vee B}(S) = W_{B \vee A}(S)$
\ \ \ \ \ \ \ \ \ \ \ \ \ \ \ \ \ \ \ \ \ \ \ \ \ \ \ \ $W_{A \wedge B}(S) = W_{B \wedge A}(S)$

$W_{A \vee A}(S) = W_{A}(S)$
\ \ \ \ \ \ \ \ \ \ \ \ \ \ \ \ \ \ \ \ \ \ \ \ \ \ \ \ \ \ \ \ \ $W_{A \wedge A}(S) = W_{A}(S)$

$W_{(A \vee B) \vee C}(S) = W_{A \vee (B \vee C)}(S)$ \ \ \ \ \ \ \ \ \ \ \ \ \  
$W_{(A \wedge B) \wedge C}(S) = W_{A \wedge (B \wedge C)}(S)$

\end{quote}
\end{proposition}
As a consequence, two equivalent formulae
have the same weight in an answer set $S$:

\begin{lemma}\label{lemma_eq_weights}
Let $A$ and $B$ be boolean combinations of the distinguished propositions $A_1, \ldots, A_n$.
If $A$ and $B$ are equivalent in propositional logic, then $W_A(S)=W_B(S)$, for all answer sets $S$. 
\end{lemma}

The preference relation associated with a complex formula can be defined by  extending condition (\ref{pref}) to any formula  $A$ which is a boolean combination of distinguished propositions. 
For all $S_1, S_2 \in \WW$, we let:

\ \ \ \ \ \ \ \ \ \ \ \ \ \ \ \ \ \ \ \ \ \ \ \  \ \ \ \ \ \ \ \ \ \ \ \  $S_1  \leq_{A}  S_2 $ \ \ iff \ \ $ W_{A}(S_1) \geq W_{A}(S_2)$.

\noindent
The preference relation $\leq_A$ is a total preorder. As a direct consequence of Lemma \ref{lemma_eq_weights}, the following proposition holds:
\begin{proposition} \label{prop_equiv_preorders}
Let $A$ and $B$ be boolean combinations of the distinguished propositions $A_1, \ldots, A_n$.
If $A$ and $B$ are equivalent in propositional logic, then,
for all answer sets $S_1, S_2 \in \WW$, $S_1\leq_A S_2$ iff $S_1 \leq_B S_2$. 
\end{proposition}

Kraus, Lehmann and Magidor \cite{KrausLehmannMagidor:90} studied the properties of several families of nonmonotonic {\em consequence relations}, i.e., binary relations $\ent$ on $\cal L$, corresponding to well-behaved sets of conditional assertions.
In particular, they introduced a notion of {\em preferential consequence relation} as a consequence relation satisfying the following properties (also called KLM postulates) that we reformulate using the typicality operator\footnote{In the KLM postulates below, $\models \mathit{A \equiv B}$ means that $A$ and $B$ are equivalent formulae in propositional logic, and $\models  \mathit{B \rightarrow  C}$ means that the implication $B \ri C$ is valid in propositional logic}:
\begin{quote}

$\mathit{(Reflexivity)} ~  \mathit{\tip(A) \ri A } $ 

$\mathit{(Right\; Weakening)}  ~ \mbox{ If } \models  \mathit{B \rightarrow  C}  \mbox{ and }       \mathit{\tip(A) \ri B}    \mbox{ then }   \mathit{ \tip(A) \ri C} $ 

$\mathit{(Left \; Logical\; Equivalence)}  ~ \mbox{ If } \models \mathit{A \equiv B}  \mbox{ and }   \mathit{\tip(A) \ri C} \mbox{ then }    \mathit{\tip(B) \ri  C} $ 

$\mathit{(And)}  ~ \mbox{ If }   \mathit{\tip(A) \ri B}  \mbox{ and }   \mathit{\tip(A) \ri C} \mbox{ then }  \mathit{\tip(A) \ri  B \wedge C} $ 

$\mathit{(Or)}  ~ \mbox{ If }  \mathit{\tip(A) \ri  C}  \mbox{ and }   \mathit{\tip(B) \ri C} \mbox{ then }    \mathit{\tip(A \vee B) \ri C} $ 

$\mathit{(Cautious \; Monotonicity)}  ~ \mbox{ If } \mathit{\tip(A) \ri B}  \mbox{ and }   \mathit{\tip(A) \ri C} \mbox{ then }  \mathit{\tip(A \wedge B) \ri C} $ 
\end{quote}
We prove that the entailment from a CondASP program $(\Pi,K)$ is a preferential consequence relation.
Let us remember that $\tip(A)  \rightarrow B$ is entailed from the program $(\Pi,K)$, \ if \ ${\emme^\Pi_K} \models \tip(A)  \rightarrow B $ holds.

\begin{proposition}
Given a consistent CondASP program $(\Pi,K)$, entailment from $(\Pi,K)$ satisfies the KLM postulates above.    
\end{proposition}
For instance, to prove that the property {\em (Or)} is satisfied, one has to prove that, assuming that ${\emme^\Pi_K} \models \mathit{\tip(A) \ri  C} $ and   ${\emme^\Pi_K} \models   \mathit{\tip(B) \ri C}$ hold, then ${\emme^\Pi_K} \models  \mathit{\tip(A \vee B) \ri C} $ also holds. 
The detailed proof of the proposition can be found in the Appendix of \cite{CondASParxiv}.

One observation is that the property of Rational Monotonicity:
\begin{quote}

$\mathit{(RM)}  ~ \mbox{ If } \mathit{\tip(A) \ri C}  \mbox{ and }   \mathit{\tip(A) \not \ri \neg B} \mbox{ then }  \mathit{\tip(A \wedge B) \ri C} $   
\end{quote}
is not satisfied by the entailment from a CondASP program $(\Pi,K)$.
In general, it is not the case that, 
if ${\emme^\Pi_K} \models \mathit{\tip(A) \ri  C} $ and   ${\emme^\Pi_K} \not \models   \mathit{\tip(A) \ri \neg B}$, it follows that ${\emme^\Pi_K} \models  \mathit{\tip(A \wedge B) \ri C} $.

As an example, suppose that from a conditional program $(\Pi,K)$, we can conclude that typical birds fly and that there are some  typical birds which are $black\& white$ (e.g., the warbler). Should we require that the most typical birds which are $black\& white$ do fly? For instance, the most typical $black\& white$ birds might be the penguins, and we do not want to conclude that they fly.
In this case, it looks reasonable that (RM) does not hold (a detailed counterexample is in the appendix in \cite{CondASParxiv}).


\section{Verifying conditional implications in ASP}
\label{sec:verif}

We describe how ASP can be used to verify entailment of
an implication
from a program $(\Pi,K)$.

In order to use meta-programming, the program $\Pi$ is transformed replacing all atoms $\mathit{A}$ with $\mathit{holds(A)}$.
A weighted defeasible implication $(\tip(A_i) \sqsubseteq B_{i,j}, h_{i,j})$ in $K$ is represented with a fact
$\mathit{wdi(A_i,B_{i,j},h_{i,j})}$.
Then, the weight with respect to an atomic formula $A$ can be computed using
$ \mathit{ \#sum \{ W1,B : wdi(A,B,W1),holds(B) \}. 
}$

For the case of conditional implications $\tip(A)  \rightarrow B$,
where $A, B$ are boolean formulae,
a solution can be adapted from \cite{JLC_SI_CILC2023}
(where an ASP encoding of entailment in a defeasible description logic with typicality was considered in the many-valued case).
The implication $\tip(A)  \rightarrow B$ to be verified is represented as
$\mathit{query(typ(A),B)}$, using, in $\mathit{A}$ and $\mathit{B}$, symbols $\mathit{and/2,or/2,neg/1}$ for representing the structure of the formulae.
The following constraint (with additional rules to extend $\mathit{holds}$ to non-atomic formulae):
\begin{quote}
    $ \mathit{ \text{:--} \ query(typ(A),B), not \ holds(A). } $
\end{quote}

\noindent
is used to select answer sets of $\Pi$ in which $A$ holds.
If there are none, the implication is trivially true.
The following weak constraint is used to prefer (at priority level 2) answer sets of $\Pi$ where the weight of $\mathit{A}$ is maximum, i.e., the answer sets in
$Min_{<_{A}}([[A]]^{{\emme^\Pi_K}})$:
\begin{quote}
    $ \mathit{ :\sim query(typ(A),B), weight(A,W). \ [-W@2] } $
\end{quote}

\noindent
The following rules and weak constraint:
\begin{quote}
    
$ \mathit{
rhs \ \text{:--} \ query(\_,B), holds(B).
}$

$ \mathit{
counterexample \ \text{:--} \ not \ rhs.
}$

$ \mathit{
:\sim counterexample. \ [-1@1]
}$
\end{quote}
are used to prefer, at level 1, the presence of 
$\mathit{counterexample}$: in this way, if there is an answer set of $\Pi$ which falsifies the implication $\tip(A)  \rightarrow B$, we get
$\mathit{counterexample}$ in all answer sets that are optimal with respect to the weak constraints; in none, otherwise. Then, a single ASP solver call computing a single optimal answer set is enough to establish entailment of the implication.

An example formula that is entailed
for the conditional ASP program in section
\ref{sec:condASP}
is:
\begin{quote}
$\tip(happy\_Sat(mary)) \ri go\_to\_cinema(mary)$
\end{quote}

General implications $A \rightarrow B$ with $\tip$ occurring in $A$ and $B$ can be dealt with as follows.
A solver call is performed for each occurrence $\tip(C)$ of $\tip$ in $A, B$, to establish, using a weak constraint, the maximum value of 
$W_{C}(S)$ across the answer sets $S$ of $\Pi$.
Then, a further solver call is performed to see if there is an answer set of $\Pi$ which is a counterexample for 
$A \rightarrow B$, evaluating the truth of formulae $\tip(C)$ with:
\begin{quote}
$ \mathit{
holds(typ(C)) \ \text{:--} \ holds(C),maxw(C,MW),weight(C,MW).
}$
\end{quote}

\noindent
where $\mathit{maxw}$ is used to represent the maximum weight computed before. This encoding provides a ${P^{NP}}$ upper bound to the complexity of entailment \cite{DBLP:conf/lpnmr/BuccafurriLR97}.

Some implications that result to be entailed
by the conditional ASP program in Section
\ref{sec:condASP}
are:
\begin{quote}

$\tip(happy\_Sat(mary)) \ri \neg \tip(happy\_Sat(bob))$

$\tip(happy\_Sat(bob)) \ri \neg \tip(happy\_Sat(mary))$

$\tip(happy\_Sat(mary) \wedge happy\_Sat(bob)) \ri \neg \tip(happy\_Sat(mary))$

$\tip(happy\_Sat(mary) \wedge happy\_Sat(bob))) \ri \tip(happy\_Sat(bob))$
\end{quote}
since Mary's and Bob's top preferences are different, and the case that maximizes weight for $happy\_Sat(mary) \wedge happy\_Sat(bob)$ is the one where they both go for pizza, which is Bob's top preference but not Mary's.
 
The two algorithms described above are implemented in dedicated reasoners based on \emph{ASP Chef}~\cite{AlvianoKR24}, a web-based platform for answer set programming. 
ASP Chef allows users to compose and execute declarative recipes, combining ASP programs with additional logic-based components, such as conditional knowledge bases.
The first algorithm, for queries of the form $\tip(A) \rightarrow B$, is available online at
\url{https://asp-chef.alviano.net/s/ICLP2025/conditional-asp-simple}.
The second one handles the general case and is accessible at
\url{https://asp-chef.alviano.net/s/ICLP2025/conditional-asp}.
In both implementations, users can specify a \emph{query}, a \emph{logic program}, and a \emph{weighted conditional knowledge base} through an intuitive web interface. 
The system elaborates the inputs by translating them into an ASP-based reasoning process, and computes the result directly in the browser. 
This makes the tool accessible without requiring any software installation, thus supporting both experimentation and education in conditional reasoning with ASP.

\section{Conclusions}
\label{sec:conclusions}

In this paper we have developed a framework for conditional reasoning from a Conditional ASP program, 
combining an ASP program $\Pi$ with a set $K$ of (weighted or ranked) conditionals over the same set of propositional variables.
The framework allows for conditional reasoning over the answer sets of the program $\Pi$, by verifying implications including the typicality operator.
We have introduced a two-valued multi-preferential semantics generalizing the KLM semantics \cite{KrausLehmannMagidor:90}
and, for weighted knowledge bases, we have developed an approach for combining preferences. The paper also provides an ASP approach for verifying entailment from a CondASP program, and corresponding reasoners in {ASP} Chef.

The idea of associating weights/ranks to the properties of concepts in weighted KBs, as a measure of their saliency, was inspired by Brewka's framework for basic preference descriptions \cite{Brewka89,Brewka04}, and by Lehmann's lexicographic closure \cite{Lehmann95}. 
Among the recent work on preference combination in ASP, let us mention  the algebraic framework for preference combination in multi-relational contextual hierarchies proposed by Bozzato et al. \cite{BozzatoEK21}, and the hybrid approach for combining an ASP program with a conditional knowledge base, proposed by Wilhelm et al. \cite{Wilhelm23} for prioritizing answer sets based on conditional expert knowledge: c-representations \cite{Kern-Isberner01} are used for ranking the feasible warehouse layouts generated by an ASP program in the logistic domain.

Our aim in this paper was to develop an integrated formalism for conditional reasoning over an extended ASP program, and to provide an ASP encoding of the formalism.
Our use of weighted knowledge bases has also some relations with 
{\em threshold concepts} in description logics by Baader et al. \cite{BaaderBGFrocos2015} and with 
{\em weighted threshold operators} by Porello et al. \cite{Porello2019}. 
The notion of typicality we have considered is also reminiscent of {\em prototype theory} \cite{Rosch1973} and  relates to Freund's ordered models  for concept representation \cite{Freund2020} although, here, we do not have a representation of concepts, but we are  considering propositions.

\medskip
{\bf Acknowledgements:} 
This research was partially supported by INDAM-GNCS.
Mario Alviano was partially supported 
by the Italian Ministry of University and Research (MUR) 
    under PRIN project PRODE ``Probabilistic declarative process mining'', CUP H53D23003420006,
    under PNRR project FAIR ``Future AI Research'', CUP H23C22000860006,
    under PNRR project Tech4You ``Technologies for climate change adaptation and quality of life improvement'', CUP H23C22000370006, and~    under PNRR project SERICS ``SEcurity and RIghts in the CyberSpace'', CUP H73C22000880001;
by the Italian Ministry of Health (MSAL)
    under POS projects CAL.HUB.RIA (CUP H53C22000800006) and RADIOAMICA (CUP H53C22000650006);
by the Italian Ministry of Enterprises and Made in Italy
    under project STROKE 5.0 (CUP B29J23000430005);
    under PN RIC project ASVIN ``Assistente Virtuale Intelligente di Negozio'' (CUP B29J24000200005);
and by the LAIA lab (part of the SILA labs). 
Mario Alviano is member of Gruppo Nazionale Calcolo Scientifico-Istituto Nazionale di Alta Matematica (GNCS-INdAM).


\begin{thebibliography}{10}
\providecommand{\bibitemdeclare}[2]{}
\providecommand{\surnamestart}{}
\providecommand{\surnameend}{}
\providecommand{\urlprefix}{Available at }
\providecommand{\url}[1]{\texttt{#1}}
\providecommand{\href}[2]{\texttt{#2}}
\providecommand{\urlalt}[2]{\href{#1}{#2}}
\providecommand{\doi}[1]{doi:\urlalt{https://doi.org/#1}{#1}}
\providecommand{\eprint}[1]{arXiv:\urlalt{https://arxiv.org/abs/#1}{#1}}
\providecommand{\bibinfo}[2]{#2}

\bibitemdeclare{article}{IJAR23}
\bibitem{IJAR23}
\bibinfo{author}{M.~\surnamestart Alviano\surnameend}, \bibinfo{author}{F.~\surnamestart Bartoli\surnameend}, \bibinfo{author}{M.~\surnamestart Botta\surnameend}, \bibinfo{author}{R.~\surnamestart Esposito\surnameend}, \bibinfo{author}{L.~\surnamestart Giordano\surnameend} \& \bibinfo{author}{\surnamestart {D. Theseider Dupr{\'{e}}}\surnameend} (\bibinfo{year}{2024}): \emph{\bibinfo{title}{A preferential interpretation of MultiLayer Perceptrons in a conditional logic with typicality}}.
\newblock {\slshape \bibinfo{journal}{Int. Journal of Approximate Reasoning}} \bibinfo{volume}{164}.
\newblock \urlprefix\url{10.1016/j.ijar.2023.109065}.

\bibitemdeclare{article}{CondASParxiv}
\bibitem{CondASParxiv}
\bibinfo{author}{M.~\surnamestart Alviano\surnameend}, \bibinfo{author}{L.~\surnamestart Giordano\surnameend} \& \bibinfo{author}{D.~Theseider \surnamestart Dupr{\'{e}}\surnameend} (\bibinfo{year}{2025}): \emph{\bibinfo{title}{A framework for Conditional Reasoning in Answer Set Programming}}.
\newblock {\slshape \bibinfo{journal}{CoRR}} \bibinfo{volume}{abs/2506.03997}, \doi{10.48550/ARXIV.2506.03997}.

\bibitemdeclare{article}{JLC_SI_CILC2023}
\bibitem{JLC_SI_CILC2023}
\bibinfo{author}{M.~\surnamestart Alviano\surnameend}, \bibinfo{author}{L.~\surnamestart Giordano\surnameend} \& \bibinfo{author}{D.~\surnamestart {Theseider Dupr{\'{e}}}\surnameend} (\bibinfo{year}{2024}): \emph{\bibinfo{title}{Complexity and Scalability of Defeasible Reasoning in Many-valued Weighted Knowledge Bases with Typicality}}.
\newblock {\slshape \bibinfo{journal}{J. Logic and Comput.}} \bibinfo{volume}{34}(\bibinfo{number}{8}), pp. \bibinfo{pages}{1469--1499}, \doi{10.1093/logcom/exae038}.

\bibitemdeclare{inproceedings}{AlvianoKR24}
\bibitem{AlvianoKR24}
\bibinfo{author}{M.~\surnamestart Alviano\surnameend} \& \bibinfo{author}{L.~Angel~Rodriguez \surnamestart Reiners\surnameend} (\bibinfo{year}{2024}): \emph{\bibinfo{title}{{ASP} Chef: Draw and Expand}}.
\newblock In: {\slshape \bibinfo{booktitle}{Proceedings of the 21st International Conference on Principles of Knowledge Representation and Reasoning, {KR}}}, \doi{10.24963/kr.2024/68}.

\bibitemdeclare{inproceedings}{rifel}
\bibitem{rifel}
\bibinfo{author}{F.~\surnamestart Baader\surnameend}, \bibinfo{author}{S.~\surnamestart Brandt\surnameend} \& \bibinfo{author}{C.~\surnamestart Lutz\surnameend} (\bibinfo{year}{2005}): \emph{\bibinfo{title}{{Pushing the} $\mathcal{EL}$ {envelope}}}.
\newblock In \bibinfo{editor}{L.P. \surnamestart Kaelbling\surnameend} \& \bibinfo{editor}{A.~\surnamestart Saffiotti\surnameend}, editors: {\slshape \bibinfo{booktitle}{Proceedings of the 19th International Joint Conference on Artificial Intelligence (IJCAI 2005)}}, \bibinfo{publisher}{Professional Book Center}, \bibinfo{address}{Edinburgh, Scotland, UK}, pp. \bibinfo{pages}{364--369}, \doi{10.1007/11551263_4}.

\bibitemdeclare{inproceedings}{BaaderBGFrocos2015}
\bibitem{BaaderBGFrocos2015}
\bibinfo{author}{F.~\surnamestart Baader\surnameend}, \bibinfo{author}{G.~\surnamestart Brewka\surnameend} \& \bibinfo{author}{O.~F. \surnamestart Gil\surnameend} (\bibinfo{year}{2015}): \emph{\bibinfo{title}{Adding Threshold Concepts to the Description Logic \emph{EL}}}.
\newblock In: {\slshape \bibinfo{booktitle}{Frontiers of Combining Systems - 10th International Symposium}}, {\slshape \bibinfo{series}{Lecture Notes in Computer Science}} \bibinfo{volume}{9322}, \bibinfo{publisher}{Springer}, pp. \bibinfo{pages}{33--48}, \doi{10.1007/978-3-319-24246-0_3}.

\bibitemdeclare{inproceedings}{BenferhatIJCAI93}
\bibitem{BenferhatIJCAI93}
\bibinfo{author}{S.~\surnamestart Benferhat\surnameend}, \bibinfo{author}{C.~\surnamestart Cayrol\surnameend}, \bibinfo{author}{D.~\surnamestart Dubois\surnameend}, \bibinfo{author}{J.~\surnamestart Lang\surnameend} \& \bibinfo{author}{H.~\surnamestart Prade\surnameend} (\bibinfo{year}{1993}): \emph{\bibinfo{title}{Inconsistency Management and Prioritized Syntax-Based Entailment}}.
\newblock In: {\slshape \bibinfo{booktitle}{Proc. IJCAI'93, Chamb{\'{e}}ry,}}, pp. \bibinfo{pages}{640--647}.

\bibitemdeclare{article}{BoothCasiniAIJ19}
\bibitem{BoothCasiniAIJ19}
\bibinfo{author}{R.~\surnamestart Booth\surnameend}, \bibinfo{author}{G.~\surnamestart Casini\surnameend}, \bibinfo{author}{T.~\surnamestart Meyer\surnameend} \& \bibinfo{author}{I.~\surnamestart Varzinczak\surnameend} (\bibinfo{year}{2019}): \emph{\bibinfo{title}{On rational entailment for Propositional Typicality Logic}}.
\newblock {\slshape \bibinfo{journal}{Artif. Intell.}} \bibinfo{volume}{277}, \doi{10.1016/j.artint.2019.103178}.

\bibitemdeclare{article}{BoothParis98}
\bibitem{BoothParis98}
\bibinfo{author}{R.~\surnamestart Booth\surnameend} \& \bibinfo{author}{J.~B. \surnamestart Paris\surnameend} (\bibinfo{year}{1998}): \emph{\bibinfo{title}{A Note on the Rational Closure of Knowledge Bases with Both Positive and Negative Knowledge}}.
\newblock {\slshape \bibinfo{journal}{Journal of Logic, Language and Information}} \bibinfo{volume}{7}(\bibinfo{number}{2}), pp. \bibinfo{pages}{165--190}, \doi{10.1023/A:1008261123028}.

\bibitemdeclare{article}{BozzatoEK21}
\bibitem{BozzatoEK21}
\bibinfo{author}{L.~\surnamestart Bozzato\surnameend}, \bibinfo{author}{T.~\surnamestart Eiter\surnameend} \& \bibinfo{author}{R.~\surnamestart Kiesel\surnameend} (\bibinfo{year}{2021}): \emph{\bibinfo{title}{Reasoning on Multirelational Contextual Hierarchies via Answer Set Programming with Algebraic Measures}}.
\newblock {\slshape \bibinfo{journal}{Theory Pract. Log. Program.}} \bibinfo{volume}{21}(\bibinfo{number}{5}), pp. \bibinfo{pages}{593--609}, \doi{10.1017/S1471068421000284}.

\bibitemdeclare{inproceedings}{Brewka89}
\bibitem{Brewka89}
\bibinfo{author}{G.~\surnamestart Brewka\surnameend} (\bibinfo{year}{1989}): \emph{\bibinfo{title}{Preferred Subtheories: An Extended Logical Framework for Default Reasoning}}.
\newblock In: {\slshape \bibinfo{booktitle}{Proceedings of the 11th International Joint Conference on Artificial Intelligence}}, pp. \bibinfo{pages}{1043--1048}.

\bibitemdeclare{inproceedings}{Brewka04}
\bibitem{Brewka04}
\bibinfo{author}{G.~\surnamestart Brewka\surnameend} (\bibinfo{year}{2004}): \emph{\bibinfo{title}{A Rank Based Description Language for Qualitative Preferences}}.
\newblock In: {\slshape \bibinfo{booktitle}{6th Europ. Conf. on Artificial Intelligence, ECAI'2004, Valencia, Spain, August 22-27, 2004}}, pp. \bibinfo{pages}{303--307}.

\bibitemdeclare{inproceedings}{BrewkaAAAI15}
\bibitem{BrewkaAAAI15}
\bibinfo{author}{G.~\surnamestart Brewka\surnameend}, \bibinfo{author}{J.~P. \surnamestart Delgrande\surnameend}, \bibinfo{author}{J.~\surnamestart Romero\surnameend} \& \bibinfo{author}{T.~\surnamestart Schaub\surnameend} (\bibinfo{year}{2015}): \emph{\bibinfo{title}{asprin: Customizing Answer Set Preferences without a Headache}}.
\newblock In: {\slshape \bibinfo{booktitle}{Proc. AAAI 2015}}, pp. \bibinfo{pages}{1467--1474}, \doi{10.1609/aaai.v29i1.9398}.

\bibitemdeclare{inproceedings}{DBLP:conf/lpnmr/BuccafurriLR97}
\bibitem{DBLP:conf/lpnmr/BuccafurriLR97}
\bibinfo{author}{Francesco \surnamestart Buccafurri\surnameend}, \bibinfo{author}{Nicola \surnamestart Leone\surnameend} \& \bibinfo{author}{Pasquale \surnamestart Rullo\surnameend} (\bibinfo{year}{1997}): \emph{\bibinfo{title}{Strong and Weak Constraints in Disjunctive Datalog}}.
\newblock In: {\slshape \bibinfo{booktitle}{{LPNMR}}}, {\slshape \bibinfo{series}{Lecture Notes in Computer Science}} \bibinfo{volume}{1265}, \bibinfo{publisher}{Springer}, pp. \bibinfo{pages}{2--17}, \doi{10.1016/0004-3702(87)90053-1}.

\bibitemdeclare{article}{Delgrande:87}
\bibitem{Delgrande:87}
\bibinfo{author}{J.~\surnamestart Delgrande\surnameend} (\bibinfo{year}{1987}): \emph{\bibinfo{title}{A First-order Conditional Logic for Prototypical Properties}}.
\newblock {\slshape \bibinfo{journal}{Artificial Intelligence}} \bibinfo{volume}{33}(\bibinfo{number}{1}), pp. \bibinfo{pages}{105--130}, \doi{10.1016/0004-3702(87)90053-1}.

\bibitemdeclare{article}{Freund2020}
\bibitem{Freund2020}
\bibinfo{author}{M.~\surnamestart Freund\surnameend} (\bibinfo{year}{2020}): \emph{\bibinfo{title}{Ordered models for concept representation}}.
\newblock {\slshape \bibinfo{journal}{J. Log. Comput.}} \bibinfo{volume}{30}(\bibinfo{number}{6}), \doi{10.1093/logcom/exaa034}.

\bibitemdeclare{book}{ASP_Gebser2012}
\bibitem{ASP_Gebser2012}
\bibinfo{author}{M.~\surnamestart Gebser\surnameend}, \bibinfo{author}{R.~\surnamestart Kaminski\surnameend}, \bibinfo{author}{B.~\surnamestart Kaufmann\surnameend} \& \bibinfo{author}{T.~\surnamestart Schaub\surnameend} (\bibinfo{year}{2012}): \emph{\bibinfo{title}{Answer Set Solving in Practice}}.
\newblock \bibinfo{series}{Synthesis Lectures on Artificial Intelligence and Machine Learning}, \bibinfo{publisher}{Morgan {\&} Claypool Publishers}, \doi{10.1007/978-3-031-01561-8}.

\bibitemdeclare{inproceedings}{Gelfond&Lifschitz:88}
\bibitem{Gelfond&Lifschitz:88}
\bibinfo{author}{Michael \surnamestart Gelfond\surnameend} \& \bibinfo{author}{Vladimir \surnamestart Lifschitz\surnameend} (\bibinfo{year}{1988}): \emph{\bibinfo{title}{The Stable Model Semantics for Logic Programming}}.
\newblock In: {\slshape \bibinfo{booktitle}{Logic Programming, Proceedings of the Fifth International Conference and Symposium}}, \bibinfo{publisher}{{MIT} Press}, pp. \bibinfo{pages}{1070--1080}.

\bibitemdeclare{article}{AIJ21}
\bibitem{AIJ21}
\bibinfo{author}{L.~\surnamestart Giordano\surnameend} \& \bibinfo{author}{V.~\surnamestart Gliozzi\surnameend} (\bibinfo{year}{2021}): \emph{\bibinfo{title}{A reconstruction of multipreference closure}}.
\newblock {\slshape \bibinfo{journal}{Artif. Intell.}} \bibinfo{volume}{290}, \doi{10.1016/j.artint.2020.103398}.

\bibitemdeclare{inproceedings}{lpar2007}
\bibitem{lpar2007}
\bibinfo{author}{L.~\surnamestart Giordano\surnameend}, \bibinfo{author}{V.~\surnamestart Gliozzi\surnameend}, \bibinfo{author}{N.~\surnamestart Olivetti\surnameend} \& \bibinfo{author}{G.~L. \surnamestart Pozzato\surnameend} (\bibinfo{year}{2007}): \emph{\bibinfo{title}{{P}referential {D}escription {L}ogics}}.
\newblock In: {\slshape \bibinfo{booktitle}{LPAR 2007}}, {\slshape \bibinfo{series}{LNAI}} \bibinfo{volume}{4790}, \bibinfo{publisher}{Springer}, \bibinfo{address}{Yerevan, Armenia}, pp. \bibinfo{pages}{257--272}, \doi{10.1007/978-3-540-75560-9_20}.

\bibitemdeclare{article}{FI09}
\bibitem{FI09}
\bibinfo{author}{L.~\surnamestart Giordano\surnameend}, \bibinfo{author}{V.~\surnamestart Gliozzi\surnameend}, \bibinfo{author}{N.~\surnamestart Olivetti\surnameend} \& \bibinfo{author}{G.~L. \surnamestart Pozzato\surnameend} (\bibinfo{year}{2009}): \emph{\bibinfo{title}{{ALC+T}: a Preferential Extension of {D}escription {L}ogics}}.
\newblock {\slshape \bibinfo{journal}{Fundamenta Informaticae}} \bibinfo{volume}{96}, pp. \bibinfo{pages}{1--32}, \doi{10.3233/FI-2009-182}.

\bibitemdeclare{article}{TPLP2016}
\bibitem{TPLP2016}
\bibinfo{author}{L.~\surnamestart Giordano\surnameend} \& \bibinfo{author}{D.~\surnamestart {Theseider Dupr{\'{e}}}\surnameend} (\bibinfo{year}{2016}): \emph{\bibinfo{title}{{ASP} for minimal entailment in a rational extension of {SROEL}}}.
\newblock {\slshape \bibinfo{journal}{{TPLP}}} \bibinfo{volume}{16}(\bibinfo{number}{5-6}), pp. \bibinfo{pages}{738--754}, \doi{10.1017/S1471068416000399}.

\bibitemdeclare{article}{TPLP2020}
\bibitem{TPLP2020}
\bibinfo{author}{L.~\surnamestart Giordano\surnameend} \& \bibinfo{author}{D.~\surnamestart {Theseider Dupr{\'{e}}}\surnameend} (\bibinfo{year}{2020}): \emph{\bibinfo{title}{An {ASP} approach for reasoning in a concept-aware multipreferential lightweight {DL}}}.
\newblock {\slshape \bibinfo{journal}{TPLP}} \bibinfo{volume}{10(5)}, pp. \bibinfo{pages}{751--766}, \doi{10.1017/S1471068420000381}.

\bibitemdeclare{inproceedings}{NMR2020}
\bibitem{NMR2020}
\bibinfo{author}{L.~\surnamestart Giordano\surnameend} \& \bibinfo{author}{D.~\surnamestart {Theseider Dupr{\'{e}}}\surnameend} (\bibinfo{year}{2020}): \emph{\bibinfo{title}{A framework for a modular multi-concept lexicographic closure semantics}}.
\newblock In: {\slshape \bibinfo{booktitle}{Proc. 18th Int. Workshop on Non-Monotonic Reasoning, NMR2020}}.

\bibitemdeclare{article}{FI_2018_sroel}
\bibitem{FI_2018_sroel}
\bibinfo{author}{Laura \surnamestart Giordano\surnameend} \& \bibinfo{author}{Daniele \surnamestart {Theseider Dupr{\'{e}}}\surnameend} (\bibinfo{year}{2018}): \emph{\bibinfo{title}{Defeasible Reasoning in {SROEL}: from Rational Entailment to Rational Closure}}.
\newblock {\slshape \bibinfo{journal}{Fundamenta Informaticae}} \bibinfo{volume}{161}(\bibinfo{number}{1-2}), pp. \bibinfo{pages}{135--161}, \doi{10.3233/FI-2018-1698}.

\bibitemdeclare{book}{Kern-Isberner01}
\bibitem{Kern-Isberner01}
\bibinfo{author}{G.~\surnamestart Kern{-}Isberner\surnameend} (\bibinfo{year}{2001}): \emph{\bibinfo{title}{Conditionals in Nonmonotonic Reasoning and Belief Revision - Considering Conditionals as Agents}}.
\newblock {\slshape \bibinfo{series}{LNCS}} \bibinfo{volume}{2087}, \bibinfo{publisher}{Springer}, \doi{10.1007/3-540-44600-1}.

\bibitemdeclare{article}{KrausLehmannMagidor:90}
\bibitem{KrausLehmannMagidor:90}
\bibinfo{author}{S.~\surnamestart Kraus\surnameend}, \bibinfo{author}{D.~\surnamestart Lehmann\surnameend} \& \bibinfo{author}{M.~\surnamestart Magidor\surnameend} (\bibinfo{year}{1990}): \emph{\bibinfo{title}{Nonmonotonic Reasoning, Preferential Models and Cumulative Logics}}.
\newblock {\slshape \bibinfo{journal}{Artificial Intelligence}} \bibinfo{volume}{44}(\bibinfo{number}{1-2}), pp. \bibinfo{pages}{167--207}, \doi{10.1016/0004-3702(90)90101-5}.

\bibitemdeclare{article}{whatdoes}
\bibitem{whatdoes}
\bibinfo{author}{D.~\surnamestart Lehmann\surnameend} \& \bibinfo{author}{M.~\surnamestart Magidor\surnameend} (\bibinfo{year}{1992}): \emph{\bibinfo{title}{What does a conditional knowledge base entail?}}
\newblock {\slshape \bibinfo{journal}{Artificial Intelligence}} \bibinfo{volume}{55}(\bibinfo{number}{1}), pp. \bibinfo{pages}{1--60}, \doi{10.1016/0004-3702(92)90041-U}.

\bibitemdeclare{article}{Lehmann95}
\bibitem{Lehmann95}
\bibinfo{author}{D.~J. \surnamestart Lehmann\surnameend} (\bibinfo{year}{1995}): \emph{\bibinfo{title}{Another Perspective on Default Reasoning}}.
\newblock {\slshape \bibinfo{journal}{Ann. Math. Artif. Intell.}} \bibinfo{volume}{15}(\bibinfo{number}{1}), pp. \bibinfo{pages}{61--82}, \doi{10.1007/BF01535841}.

\bibitemdeclare{book}{Lewis:73}
\bibitem{Lewis:73}
\bibinfo{author}{D.~\surnamestart Lewis\surnameend} (\bibinfo{year}{1973}): \emph{\bibinfo{title}{Counterfactuals}}.
\newblock \bibinfo{publisher}{Basil Blackwell Ltd}.

\bibitemdeclare{inproceedings}{Makinson88}
\bibitem{Makinson88}
\bibinfo{author}{David \surnamestart Makinson\surnameend} (\bibinfo{year}{1988}): \emph{\bibinfo{title}{General Theory of Cumulative Inference}}.
\newblock In: {\slshape \bibinfo{booktitle}{Non-Monotonic Reasoning, 2nd International Workshop}}, pp. \bibinfo{pages}{1--18}, \doi{10.1007/3-540-50701-9_16}.

\bibitemdeclare{article}{Nute80}
\bibitem{Nute80}
\bibinfo{author}{D.~\surnamestart Nute\surnameend} (\bibinfo{year}{1980}): \emph{\bibinfo{title}{Topics in Conditional Logic}}.
\newblock {\slshape \bibinfo{journal}{Reidel, Dordrecht}}, \doi{10.1007/978-94-009-8966-5}.

\bibitemdeclare{book}{Pearl:88}
\bibitem{Pearl:88}
\bibinfo{author}{J.~\surnamestart Pearl\surnameend} (\bibinfo{year}{1988}): \emph{\bibinfo{title}{Probabilistic Reasoning in Intelligent Systems Networks of Plausible Inference}}.
\newblock \bibinfo{publisher}{Morgan Kaufmann}.

\bibitemdeclare{inproceedings}{Pearl90}
\bibitem{Pearl90}
\bibinfo{author}{J.~\surnamestart Pearl\surnameend} (\bibinfo{year}{1990}): \emph{\bibinfo{title}{System {Z:} {A} Natural Ordering of Defaults with Tractable Applications to Nonmonotonic Reasoning}}.
\newblock In: {\slshape \bibinfo{booktitle}{TARK'90, Pacific Grove, CA, USA}}, pp. \bibinfo{pages}{121--135}.

\bibitemdeclare{inproceedings}{Porello2019}
\bibitem{Porello2019}
\bibinfo{author}{D.~\surnamestart Porello\surnameend}, \bibinfo{author}{O.~\surnamestart Kutz\surnameend}, \bibinfo{author}{G.~\surnamestart Righetti\surnameend}, \bibinfo{author}{N.~\surnamestart Troquard\surnameend}, \bibinfo{author}{P.~\surnamestart Galliani\surnameend} \& \bibinfo{author}{C.~\surnamestart Masolo\surnameend} (\bibinfo{year}{2019}): \emph{\bibinfo{title}{A Toothful of Concepts: Towards a Theory of Weighted Concept Combination}}.
\newblock In: {\slshape \bibinfo{booktitle}{Proceedings of the 32nd International Workshop on Description Logics}}, {\slshape \bibinfo{series}{{CEUR} Workshop Proceedings}} \bibinfo{volume}{2373}, \bibinfo{publisher}{CEUR-WS.org}.

\bibitemdeclare{article}{Rosch1973}
\bibitem{Rosch1973}
\bibinfo{author}{Eleanor \surnamestart Rosch\surnameend} (\bibinfo{year}{1973}): \emph{\bibinfo{title}{Natural Categories}}.
\newblock {\slshape \bibinfo{journal}{Cognitive Psychology}} \bibinfo{volume}{4}(\bibinfo{number}{3}), pp. \bibinfo{pages}{328--350}, \doi{10.1016/0010-0285(73)90017-0}.

\bibitemdeclare{inproceedings}{Wilhelm23}
\bibitem{Wilhelm23}
\bibinfo{author}{M.~\surnamestart Wilhelm\surnameend}, \bibinfo{author}{A.~\surnamestart Thevapalan\surnameend} \& \bibinfo{author}{G.~\surnamestart Kern{-}Isberner\surnameend} (\bibinfo{year}{2023}): \emph{\bibinfo{title}{Prioritizing Answer Sets Based on Conditional Expert Knowledge}}.
\newblock In: {\slshape \bibinfo{booktitle}{{FLAIRS 2023}}}, \doi{10.32473/FLAIRS.36.133167}.

\end{thebibliography}

\end{document}